\documentclass{article}

\usepackage{microtype}
\usepackage{graphicx}
\usepackage{subfigure}
\usepackage{booktabs} 

\usepackage{hyperref}

\usepackage[accepted]{icml2025}

\usepackage{amsmath}
\usepackage{amssymb}
\usepackage{mathtools}
\usepackage{amsthm}

\usepackage[capitalize,noabbrev]{cleveref}

\theoremstyle{plain}
\newtheorem{theorem}{Theorem}[section]

\theoremstyle{definition}

\theoremstyle{remark}

\usepackage[textsize=tiny]{todonotes}

\usepackage{bbding}
\usepackage{pifont}
\usepackage{wrapfig}
\usepackage{scalefnt}
\usepackage{scalerel}
\usepackage{tikz}
\usepackage{fontawesome}
\usepackage{soul}

\usepackage{multirow}

\definecolor{secolor}{rgb}{0.4,0.4,0.4}

\usepackage{colortbl}
\definecolor{Gray}{gray}{0.925}

\definecolor{shadecolorab}{rgb}{0.95,0.95,0.95}

\usepackage{enumitem}
\setlist{leftmargin=3mm}

\icmltitlerunning{Divide and Conquer: Accelerating Diffusion-Based Large Language Models via Adaptive Parallel Decoding}

\begin{document}

\twocolumn[
\icmltitle{Divide and Conquer: Accelerating Diffusion-Based \\ Large Language Models via Adaptive Parallel Decoding}

\icmlsetsymbol{equal}{*}

\begin{icmlauthorlist}
\icmlauthor{Xiangzhong Luo}{seu}
\icmlauthor{Yilin An}{seu}
\icmlauthor{Zhicheng Yu}{seu}
\icmlauthor{Weichen Liu}{ntu}
\icmlauthor{Xu Yang}{seu}
\end{icmlauthorlist}

\icmlaffiliation{seu}{Southeast University}
\icmlaffiliation{ntu}{Nanyang Technological University}
\vspace{0.3in}
]

\printAffiliationsAndNotice 

\begin{abstract}
Diffusion-based large language models (dLLMs) have shown promising performance across various reasoning tasks, establishing themselves as an alternative to autoregressive large language models (LLMs). Unlike autoregressive LLMs that generate one token per step based on all previous tokens, dLLMs theoretically enable parallel generation of multiple tokens at each decoding step. However, recent dLLMs still favor one-token-per-step generation in practice, as directly decoding multiple masked tokens often leads to degraded generation quality and stability. This reveals a substantial gap between the \textit{theoretical parallelism} and \textit{practical performance} of dLLMs. To bridge this gap, we introduce an adaptive parallel decoding approach, namely \textit{DiCo}, which features a three-phase \textit{divide-and-conquer} paradigm to unleash the inherent parallelism of dLLMs. \ding{182} During the \textit{Divide} phase, \textit{DiCo} first explores the input masked sequence and identifies masked tokens as seed tokens, which are then expanded to construct a set of local clusters. \ding{183} During the \textit{Conquer} phase, \textit{DiCo} performs parallel decoding across different local clusters as constructed in the \textit{Divide} phase. The \textit{divide–and-conquer} process repeatedly alternates between the \textit{Divide} and \textit{Conquer} phases until convergence. \ding{184} During the \textit{Finalize} phase, \textit{DiCo} decodes the remaining few masked tokens using an effective fine-grained compound decoding scheme to finalize the generation. Extensive experiments demonstrate that \textit{DiCo} can deliver significant inference speedups while maintaining competitive generation quality.

\end{abstract}

\section{Introduction}
\label{sec:introduction}

Large language models (LLMs) have recently achieved impressive success across various real-world scenarios, such as reasoning, text generation, and information retrieval~\cite{dubey2024llama, achiam2023gpt, guo2025deepseek}. Despite their success, current state-of-the-art (SOTA) LLMs largely rely on autoregressive (AR) decoding, where each token must be generated in sequence based on all previous tokens (i.e., one token per step)~\cite{li2024survey}. This inherently sequential process inevitably suffers from limited parallelism, which translates into high latency and low throughput~\cite{li2024large}.

A recent line of research to tackle this dilemma is diffusion-based large language models (dLLMs)~\cite{sahoo2024simple, nie2025large, ye2025dream}. In contrast to AR LLMs, these dLLMs feature an iterative denoising paradigm to generate text over discrete decoding steps, where a fully masked initial sequence is progressively transformed into a fully unmasked final output~\cite{nie2025large}. This paradigm theoretically enables parallel generation of multiple tokens at each decoding step, which thus offers a promising pathway towards substantially faster inference while maintaining strong generation quality~\cite{li2025survey-dllm}.

\begin{figure}[t]
  \centering
  \includegraphics[width=1.0\linewidth]{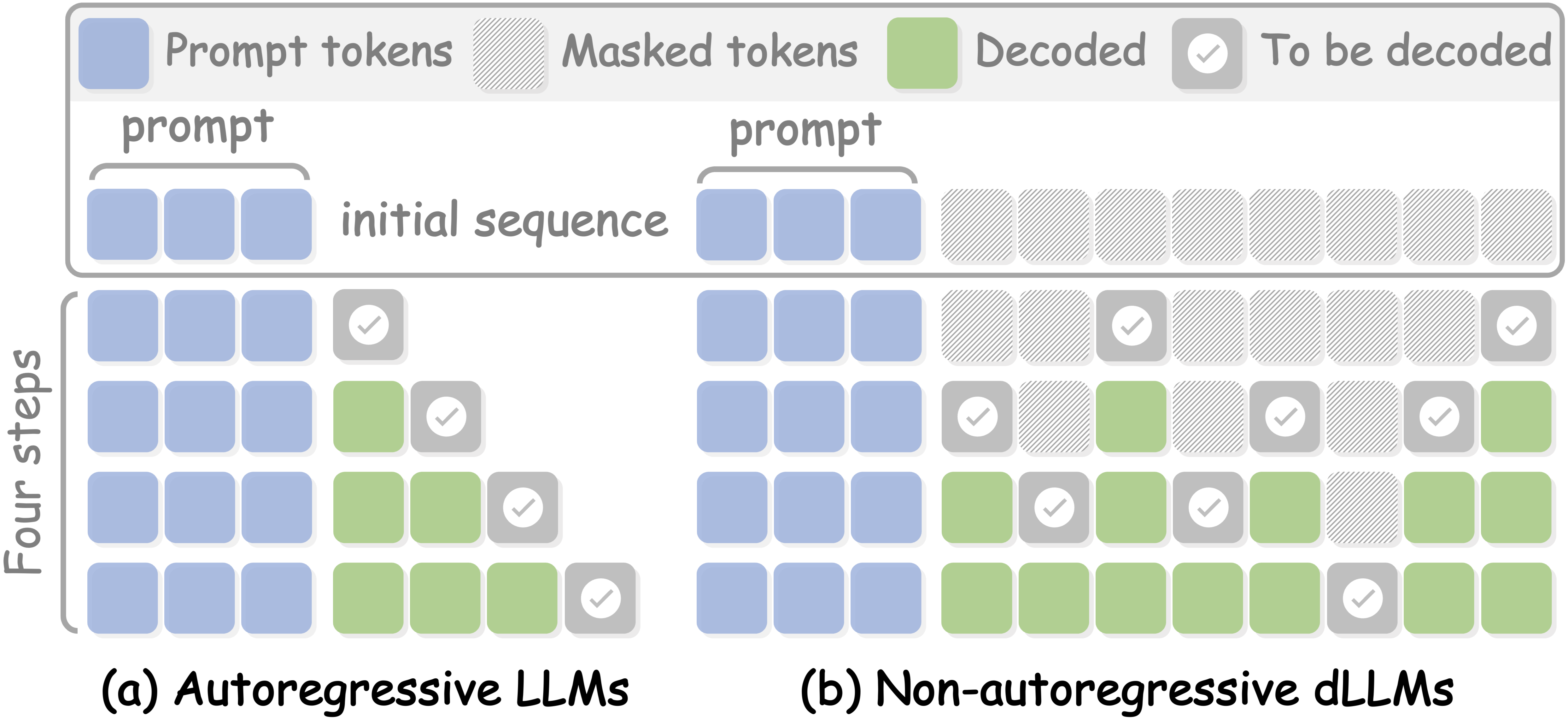} 
  \vspace{-10pt}
  \caption{Comparisons between autoregressive LLMs and non-autoregressive dLLMs, where each box denotes a specific token.}
  \label{fig:llm-vs-dllm} 
\end{figure}

Despite their theoretical potential for parallel generation, current representative dLLMs~\cite{nie2025large, ye2025dream} still favor generating tokens one at a time (i.e., one token per decoding step) in practice to maintain generation quality and stability compared to AR LLMs~\cite{israel2025accelerating}. Our preliminary experiments further confirm that the best generation quality is typically achieved when dLLMs update only one token at each decoding step. In contrast, directly leveraging the inherent parallelism of dLLMs to generate multiple tokens per decoding step often leads to degraded generation quality~\cite{wu2025fast}. These findings reveal a substantial gap between the theoretical parallelism and practical performance of dLLMs.

To bridge the above gap, we propose a \textbf{training-free} adaptive parallel decoding approach, dubbed \textit{DiCo}, which leverages a \textit{divide-and-conquer} paradigm to unleash the inherent parallelism of dLLMs and accelerate dLLM inference. \textit{DiCo} consists of the following three phases. \ding{182} During the \textit{Divide} phase, \textit{DiCo} first explores the input sequence and identifies a subset of masked tokens as seed tokens, which are then expanded to form a set of local clusters. \ding{183} During the \textit{Conquer} phase, \textit{DiCo} performs adaptive parallel decoding across different local clusters. The \textit{divide–and-conquer} process repeatedly alternates between the \textit{Divide} and \textit{Conquer} phases until convergence. \ding{184} During the \textit{Finalize} phase, \textit{DiCo} decodes the remaining few masked tokens using an effective fine-grained compound decoding scheme to finalize the generation. Extensive experiments demonstrate that \textit{DiCo} can deliver significant inference speedups while maintaining competitive generation quality.

To summarize, this paper makes the following contributions:
\begin{itemize}
    \item {
    \textbf{Motivation.}
    Our theoretical analysis and empirical results reveal a substantial gap between the theoretical parallelism and practical performance of dLLMs, which motivates the necessity to develop an effective parallel decoding approach for dLLMs to bridge this gap.
    }
    \item {
    \textbf{Solution.}
    We introduce \textit{DiCo}, a training-free adaptive parallel decoding approach that leverages a \textit{divide-and-conquer} paradigm to iteratively perform parallel decoding across the input masked sequence, which can deliver significant inference speedups while also maintaining competitive generation quality and stability.
    }
    \item {      
    \textbf{Evaluation.}
    We conduct extensive experiments to evaluate \textit{DiCo} on various representative dLLMs and datasets, which clearly show the efficacy of \textit{DiCo} in terms of both inference efficiency and generation quality.
    }
\end{itemize}

\section{Related Works}
\label{sec:related-works}

\textbf{Diffusion-based large language models.}
Building on diffusion models in continuous domains~\cite{croitoru2023diffusion-survey}, recent studies have extended diffusion models to discrete language modeling~\cite{sahoo2024simple, shi2024simplified, nie2024scaling}. Unlike AR LLMs that generate tokens sequentially~\cite{dubey2024llama, achiam2023gpt, guo2025deepseek}, dLLMs feature an iterative denoising process, which can enable bidirectional context modeling and theoretically support parallel decoding~\cite{li2025survey-dllm}. More recently, several dLLMs, such as LLaDA~\cite{nie2025large, zhu2025llada-v1.5, zhu2025llada-moe} and Dream~\cite{ye2025dream}, have shown competitive performance across various tasks, establishing themselves as a promising alternative to AR LLMs. However, directly leveraging the theoretical parallelism of dLLMs to generate multiple tokens in parallel often leads to degraded generation quality in practice~\cite{wu2025fast}. This reveals a substantial gap between the theoretical parallelism and practical performance of dLLMs.

\textbf{Parallel decoding for dLLMs.}
Recent efforts have explored parallel decoding for dLLMs, which can be divided into training-based~\cite{bao2025learning, chen2025dparallel, wu2025fast-dllmv2} and training-free~\cite{wu2025fast, wu2025free, gao2025self} categories. Among them, \cite{bao2025learning, chen2025dparallel, wu2025fast-dllmv2} introduce additional training to learn decoding policies for dLLMs, which, however, incurs substantial training cost and greatly limits their applicability. In addition, \cite{wu2025fast, wu2025free, gao2025self} follow block-wise decoding---also known as semi-autoregressive (semi-AR) decoding---and propose tailored training-free parallel decoding strategies. However, these training-free parallel decoding methods still suffer from inherent limitations due to their block-wise nature~\cite{huang2025pc}. More recently, \cite{wu2025free, gao2025self} integrate speculative decoding~\cite{leviathan2023fast} into dLLMs to enable parallel generation. Nonetheless, since dLLMs do not naturally support batch inference~\cite{zhu2025llada-moe}, these speculative decoding-based methods fail to deliver realistic inference speedups.

\textbf{Approximate KV Caching for dLLMs.}
Unlike AR LLMs, dLLMs employ bidirectional attention, which prevents them from directly leveraging the standard KV cache~\cite{li2025survey-dllm}. To mitigate this limitation, recent work has found that the KV states in dLLMs exhibit strong temporal consistency across adjacent decoding iterations. Based on this finding, several approximate KV caching methods~\cite{liu2025dllm-cache, ma2025dkv-cache, wu2025fast, jiang2025d} have been proposed for dLLMs. For instance, dLLM-Cache~\cite{liu2025dllm-cache} separates the sequence into a prompt portion and a generated portion, and refreshes their corresponding KV states with different update schedules. dKV-Cache~\cite{ma2025dkv-cache} adopts a delayed update strategy, where the KV states of newly generated tokens are recorded in the following decoding step. Fast-dLLM~\cite{wu2025fast} combines block-wise semi-AR decoding with selective caching, retaining KV states for all tokens except those within the actively decoded block.
\section{Preliminaries}
\label{sec:preliminaries}

\subsection{Generation Process of dLLMs}
\label{sec:generation-process-of-dllms}

\begin{figure}[t]
  \centering
  \includegraphics[width=1.0\linewidth]{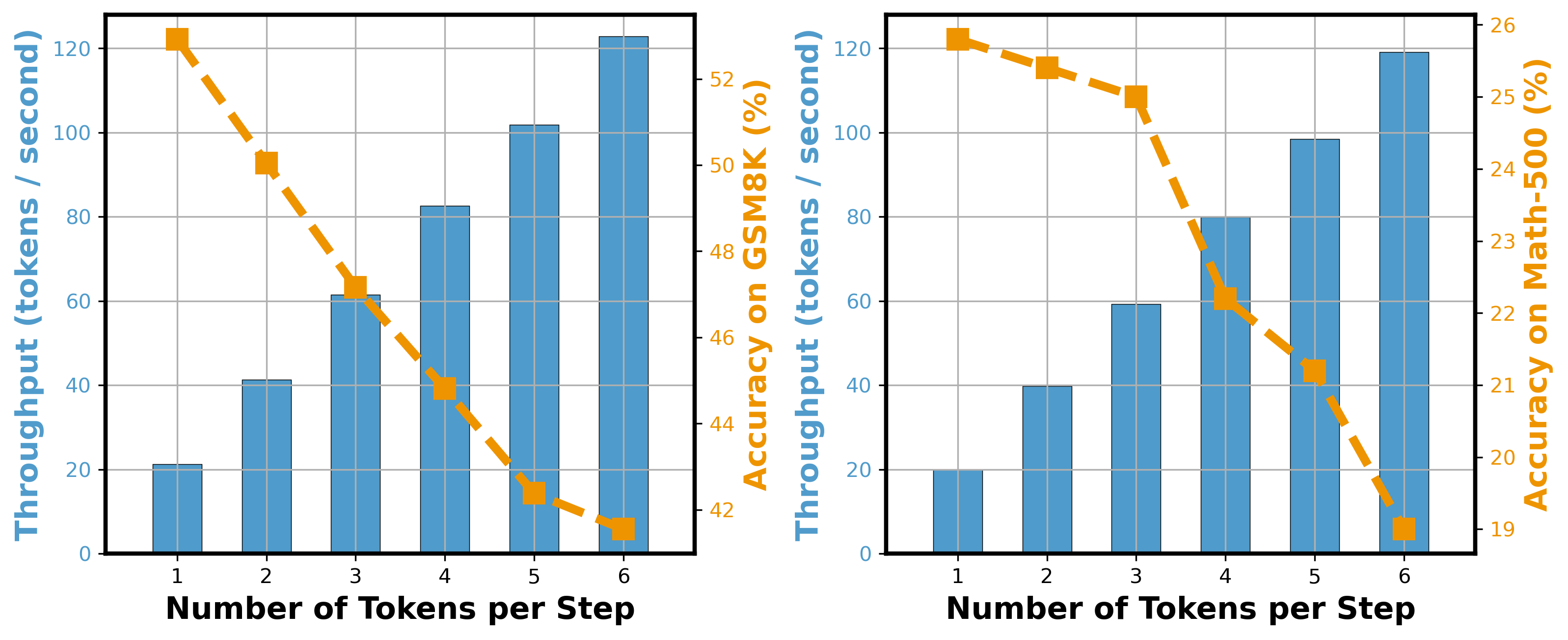} 
  \vspace{-10pt}
  \caption{Illustration of the performance collapse of naive parallel decoding on GSM8K (\textit{left}) and Math-500 (\textit{right}).}
  \label{fig:parallel-decoding-observation} 
\end{figure}

Diffusion-based large language models (dLLMs) leverage an iterative denoising paradigm to generate text over $T$ discrete steps, where a fully masked initial sequence is gradually transformed into a fully unmasked final output~\cite{li2025survey-dllm}. Formally, we use $\mathcal{V}$ to denote the vocabulary, which includes a special masked token $\texttt{[MASK]}$. The generation process begins with an initial sequence $\mathbf{x}^{T} \in \mathcal{V}^L$ of length $L$, which consists of a static prompt segment $\mathbf{p} = [p_0, \ldots, p_{m-1}]$ of $m$ tokens and a dynamic response segment $\mathbf{r}_0 = [\texttt{[MASK]}, \ldots, \texttt{[MASK]}]$ of $n = L - m$ masked tokens as follows:
\begin{equation}
\mathbf{x}^{T} = (\mathbf{p}, \mathbf{r}) = [p_0, \ldots, p_{m-1}, \texttt{[MASK]}_{0}, \ldots, \texttt{[MASK]}_{n-1}].
\label{eq:input-masked-sequence}
\end{equation}
At each step $t \in \{T-1, \ldots, 0\}$, a mask predictor $f_{\theta}(\cdot)$—typically a transformer with bidirectional attention—is employed to estimate the output distribution for each position in the sequence as $\mathbf{y}^t=f_{\theta}(\mathbf{x}^t)$. Then, greedy decoding is performed on $\mathbf{y}^t$ to derive the predicted tokens $\hat{\mathbf{x}}^t$ for all masked positions as follows:
\begin{equation}
\hat{\mathbf{x}}_i^t 
= \mathop{\text{arg\,max}}\nolimits_{v \in \mathcal{V}} \; \text{Softmax}(\mathbf{y}_i^t), 
\quad \text{if } \mathbf{x}_i^t = \texttt{[MASK]}.
\end{equation}
Furthermore, the transition function $\mathcal{G}(\cdot)$ selectively updates masked tokens in $\mathbf{x}^t$ based on $\hat{\mathbf{x}}^t$ to generate $\mathbf{x}^{t-1}$ as:
\begin{equation}
\mathbf{x}^{t-1} = \mathcal{G}(\mathbf{x}^t, \hat{\mathbf{x}}^t), \; \text{where} \; \mathbf{x}_i^{t-1} = \begin{cases}
        \hat{\mathbf{x}}_i^t, & \text{if }\; i \in S_t \\
        \mathbf{x}_i^t, & \text{otherwise}
    \end{cases}
\end{equation}
where $S_t$ is a subset of masked positions sampled from $\mathbf{x}^t$ using a specific sampling strategy (e.g., random or confidence-based sampling~\cite{nie2025large}). After $T$ steps, the output sequence $\mathbf{x}^0$ contains no masked tokens, marking the completion of the generation process~\cite{nie2025large}.

\subsection{Theoretical Inspirations}
\label{sec:theoretical-inspirations}

Although dLLMs theoretically enable parallel generation of multiple tokens within a single decoding step, recent dLLMs~\cite{nie2025large, ye2025dream} still tend to generate one token per step in practice in order to preserve generation quality and stability. This discrepancy reveals a substantial gap between the \textit{theoretical parallelism} and \textit{practical performance} of dLLMs. Below we first provide a theoretical analysis on this gap.

\begin{figure}[t]
  \centering
  \includegraphics[width=1.0\linewidth]{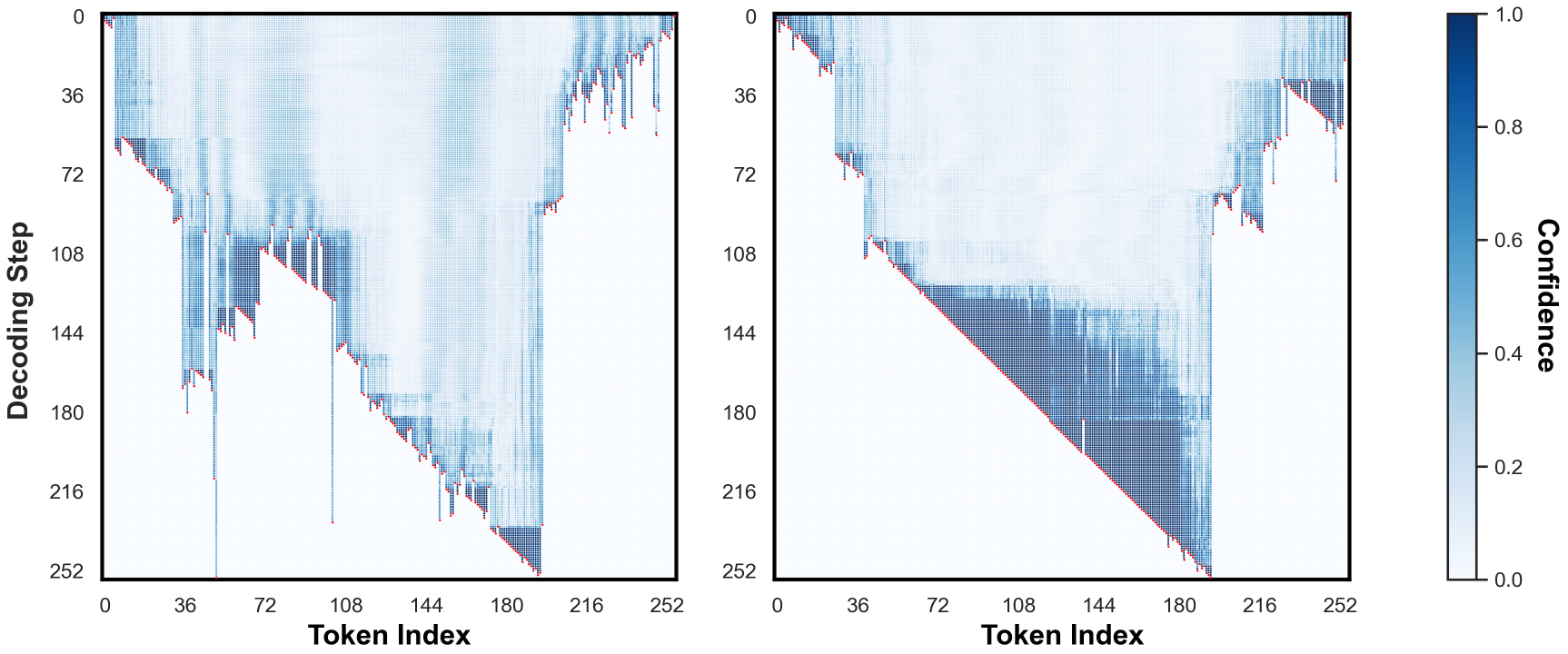} 
  \vspace{-10pt}
  \caption{Visualization of two representative decoding trajectories of LLaDA-8B-Instruct on GSM8K. The blue points represent the predicted confidence of masked tokens, while the red points denote unmasked tokens. (best viewed when zoomed in)}
  \label{fig:decoding-cluster-observation} 
\end{figure}

\textbf{Pitfalls of naive parallel decoding.}
Recent dLLMs~\cite{nie2025large, ye2025dream} tend to generate one token per step to preserve generation quality, which can be interpreted as modeling the decoding trajectory $\mathbf{x}$ via its joint probability mass function (PMF) over all decoded tokens:
\begin{equation}
    P_{joint}(\mathbf{x}) = P(\mathbf{x}_0) \prod\nolimits_{i=1}^{n-1} P(\mathbf{x}_{i} | \mathbf{x}_0, \ldots, \mathbf{x}_{i-1}),
    \label{eq:joint-PMF}
\end{equation}
where $P(\mathbf{x}_{i} | \mathbf{x}_0, \ldots, \mathbf{x}_{i-1})$ is the conditional probability of generating the $i$-th token based on all previously decoded tokens. However, in naive parallel decoding, the decoding trajectory $\mathbf{x}$ is approximated as the product of the marginal PMFs over all decoded tokens:
\begin{equation}
    P_{marginal}(\mathbf{x}) = \prod\nolimits_{i=0}^{n-1} P(\mathbf{x}_{i}),
    \label{eq:marginal-PMF}
\end{equation}
where $P(\mathbf{x}_i)$ corresponds to the marginal PMF of the $i$-th token. This formulation implicitly assumes that all tokens are mutually independent, i.e., $P(\mathbf{x}_0, \ldots, \mathbf{x}_i) = P(\mathbf{x}_{i} | \mathbf{x}_0, \ldots, \mathbf{x}_{i-1})$. However, this assumption rarely holds in natural language generation, where strong contextual dependence exists among adjacent tokens within local context. As a result, naive parallel decoding may deviate from the true data distribution, thus producing incoherent or low-quality text~\cite{wu2025fast, wu2025fast-dllmv2}.

{\faLightbulbO}
\textbf{Inspirations.}
The above theoretical analysis suggests that the fundamental limitation of naive parallel decoding arises from its implicit assumption of independence among tokens. Fortunately, this also indicates that, when certain masked tokens are conditionally independent---their generation does not depend on other masked tokens in the context---it becomes theoretically valid to decode them in parallel without deviating from the true data distribution.

\subsection{Empirical Observations}
\label{sec:empirical-observations}

\begin{figure*}[t]
  \centering
  \includegraphics[width=1.0\linewidth]{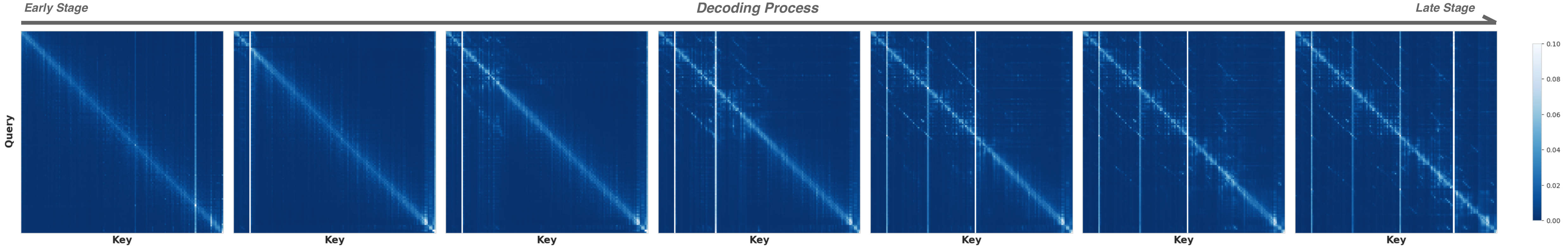} 
  \vspace{-10pt}
  \caption{Visualization of attention patterns in LLaDA-8B-Instruct during the decoding process (accumulated on GSM8K).}
  \label{fig:token-dependency-observation} 
\end{figure*}

\textbf{Observation I.}
\textit{The naive parallel decoding in dLLMs suffers from substantial accuracy loss.}
As shown in Section~\ref{sec:theoretical-inspirations}, we analyze the pitfalls of naive parallel decoding from a theoretical perspective. Below we further present experimental evidence to support the above theoretical analysis. Specifically, we take LLaDA-8B-Instruct as an example and evaluate its performance on two math reasoning tasks under zero-shot settings, where the generation length is fixed to 256. During this process, multiple masked tokens are decoded at each step following a top-$k$ confidence-based parallel decoding scheme~\cite{nie2025large}. As shown in Figure~\ref{fig:parallel-decoding-observation}, naive parallel decoding, despite being able to enhance inference throughput, suffers from significant accuracy loss on both GSM8K and Math-500.

\textbf{Observation II.}
\textit{The dependency among masked tokens in dLLMs evolves from local sparsity to global density as decoding progresses.}
To understand the decoding behaviors of dLLMs, we begin by analyzing their attention patterns, where we consider LLaDA-8B-Instruct and visualize its attention patterns on GSM8K during decoding. As shown in Figure~\ref{fig:token-dependency-observation}, the dependency among masked tokens remains highly sparse in the early and middle decoding stages---masked tokens largely rely on their local context and exhibit weak correlations with distant ones. As decoding progresses, the dependency among masked tokens gradually becomes denser, where masked tokens increasingly attend to a broader range of global context rather than only their local neighborhoods. These empirical findings reveal that early and middle decoding stages allow for higher degrees of parallelism due to locally sparse dependency, whereas the late stage demands more context-aware or fine-grained decoding to maintain generation quality.

\textbf{Observation III.}
\textit{Masked tokens with moderate yet stable confidence tend to rapidly evolve into distinct clusters during decoding.} 
To further shed light on the decoding trajectories of dLLMs, we visualize two representative decoding trajectories of LLaDA-8B-Instruct on GSM8K in Figure~\ref{fig:decoding-cluster-observation}, from which we derive two key observations. First, masked tokens with moderate yet stable confidence tend to form clusters that quickly evolve into high-confidence zones within a few decoding steps. As decoding progresses, these zones expand as confidence propagates to neighboring tokens that previously exhibit moderate confidence. Second, many masked tokens with moderate confidence are overlooked in the early and middle decoding stages due to the conservativeness of top-1 confidence-based decoding. According to our previous observation, masked tokens within these distant clusters exhibit negligible dependence on masked tokens outside their own clusters, which suggests that it is safe to decode them in parallel.

\subsection{Motivations}
\label{sec:motivations}

Our theoretical analysis in Section~\ref{sec:theoretical-inspirations}, together with our empirical findings in Observation~I, reveals the inherent limitations of naive parallel decoding for dLLMs. Furthermore, as shown in Observation~III, the decoding trajectories of dLLMs naturally exhibit a clustering behavior, where masked tokens with moderate yet stable confidence tend to form local clusters. Fortunately, as shown in Observation~II, these local clusters exhibit strong internal dependence but weak external dependence during the early and middle decoding stages. Taken together, these theoretical and empirical findings suggest that we can leverage a \textit{divide-and-conquer} paradigm to partition the input masked sequence into multiple local clusters and then safely perform parallel decoding across these local clusters during the early and middle decoding stages. Finally, for the remaining masked tokens---typically few in number and exhibiting strong inter-token dependence---we can tailor an effective fine-grained decoding strategy to maintain generation quality.
\section{\textit{DiCo}: Divide and Conquer}
\label{sec:method}

\subsection{Overview}
\label{sec:overview}

Building on the above theoretical inspirations and empirical observations, we propose \textit{DiCo}, a \textbf{training-free} adaptive parallel decoding approach to accelerate dLLM inference without compromising generation quality and stability. Specifically, \textit{DiCo} features a \textit{divide-and-conquer} paradigm to decode masked tokens in three phases.
\ding{182} During the \textit{Divide} phase, \textit{DiCo} first explores the input sequence and identifies a subset of masked tokens as seed tokens, which are then bidirectionally expanded to form local clusters. These local clusters are further refined and merged into larger non-overlapping clusters for the subsequent \textit{Conquer} phase. 
\ding{183} During the \textit{Conquer} phase, \textit{DiCo} adaptively decodes masked tokens across different local clusters in parallel. The \textit{divide–and-conquer} process repeatedly alternates between the \textit{Divide} and \textit{Conquer} phases until convergence.
\ding{184} During the \textit{Finalize} phase, \textit{DiCo} decodes the remaining masked tokens---which are few in number and exhibit strong inter-token dependence---using an effective tailored decoding strategy to maintain generation quality.

\subsection{Phase 1: How to Divide}
\label{sec:phase1-how-to-divide}

In the \textit{Divide} phase, \textit{DiCo} formulates dLLM decoding into an $N$-way \textit{divide–and–conquer} problem, in which the input masked sequence is progressively partitioned into multiple local clusters. This process involves several exploratory iterations, each of which consists of two alternating steps: \ul{(1)} identifying seed tokens and \ul{(2)} expanding and merging local clusters. These two steps are iteratively repeated to optimize and stabilize the partitioning process.

\textbf{Identifying seed tokens.}
The prediction confidence naturally serves as an effective proxy for identifying seed tokens, as it captures essential semantic and contextual information. However, due to the inductive biases of Transformer architectures, the confidence distribution in dLLMs tends to be highly localized~\cite{nie2025large}, making it challenging to explore seed tokens across the entire masked sequence. To address this issue, we draw inspirations from Soft-NMS~\cite{softnms2017} and leverage its principle to enforce spatial separation among different seed tokens. Specifically, we introduce a position-aware function $D(\cdot)$ to discourage \textit{DiCo} from identifying seed tokens at neighboring positions, which can be formulated as follows:
\begin{equation}
    D(i,j)=1-\phi (|i-j|)=1 - \exp \left (-\frac{|i-j|^2}{2\sigma^2} \right ),
\end{equation}
where $\phi(\cdot)$ is a Gaussian function, $i$ is the position of an already identified seed token, and $j$ is another masked position. For the Gaussian function $\phi(\cdot)$, we empirically calculate its standard deviation as $\sigma=(L-m)/3\sqrt{2}$, where $L$ is the generation length and $m$ is the prompt length.

Furthermore, to handle order-sensitive reasoning tasks that favor left-to-right generation~\cite{huang2025pc}, we introduce a trajectory-guided scheme to adaptively adjust the prediction confidence of each masked token. Taken together, the calibrated confidence for each masked token at step $t$ can be formulated as follows:
\begin{equation}
    c^{t}_{*|i}(j|i) = c^t_w(j) \cdot D(i, j)  = \left[c^t(j) \cdot W(j, t)\right] \cdot D(i, j),
\end{equation}
where $c^t(j)$ denotes the original prediction confidence. The trajectory-guided term $W(\cdot)$ is defined as follows:
\begin{equation}
W(j,t) = \exp\!\left(\frac{j \cdot \ln(\text{clip}(\alpha \cdot R(\mathbf{x}^t) + \beta,\,0,\,1))}{L-m}\right),
\end{equation}
where $\alpha$ and $\beta$ are constant coefficients, and $R(\cdot)$ denotes the unmasking ratio at step $t$. With the above in place, \textit{DiCo} iteratively selects seed tokens based on their top-1 calibrated confidence. We note that, once a seed token is selected, it is accepted only if its trajectory-guided confidence $c_w^t(j)$ exceeds the moderate threshold $\tau_1$. As such, the accepted seed tokens are both semantically rich and spatially separated, allowing \textit{DiCo} to safely unmask them within a single decoding step to advance the exploratory iterations.

\textbf{Expanding and merging local clusters.}
After a seed token is identified, \textit{DiCo} proceeds to form its local cluster through a series of bidirectional expansions. Each expansion incorporates neighboring masked tokens whose trajectory-guided confidence exceeds the threshold $\tau_1$. For each seed token, its expansion begins from the boundaries of an existing cluster if the seed token lies within it. Otherwise, the expansion starts from the seed token itself to form a new cluster. During this process, the adjacent clusters with overlapping are also merged into larger clusters. These clusters converge under two conditions: either their trajectory-guided confidence density exceeds the threshold $\tau_{2}$ for parallel decoding, or a maximum number of iterations $T_{max}$ has been reached. Upon convergence, \textit{DiCo} terminates the \textit{Divide} phase and passes the resulting local clusters to the subsequent \textit{Conquer} phase for parallel decoding.

\begin{table*}[t]
\caption{Comprehensive evaluation results on LLaDA-8B-Instruct and Dream-7B-Instruct under both non-AR and semi-AR settings. \textbf{Bold} numbers indicate the best results and numbers in brackets denote the corresponding speedup ratios compared to the Vanilla baseline.}
\centering
\resizebox{1.0\linewidth}{!}{
\begin{tabular}{llcccccccc}
\toprule[0.125em]
\multicolumn{1}{l}{\multirow{3}{*}{\textbf{Dataset}}} & \multirow{3}{*}{\textbf{Method}} & \multicolumn{4}{c}{\textbf{LLaDA-8B-Instruct~\cite{nie2025large}}} & \multicolumn{4}{c}{\textbf{Dream-7B-Instruct~\cite{ye2025dream}}} \\ 
\cmidrule(lr){3-6} \cmidrule(lr){7-10} 
\multicolumn{1}{c}{} & & \multicolumn{2}{c}{\textbf{non-AR}} & \multicolumn{2}{c}{\textbf{semi-AR}} & \multicolumn{2}{c}{\textbf{non-AR}} & \multicolumn{2}{c}{\textbf{semi-AR}} \\ 
\cmidrule(lr){3-4} \cmidrule(lr){5-6} \cmidrule(lr){7-8} \cmidrule(lr){9-10}
\multicolumn{1}{c}{} & & \multicolumn{1}{c}{\textbf{Accuracy}} & \multicolumn{1}{c}{\textbf{Throughput}} & \multicolumn{1}{c}{\textbf{Accuracy}} & \multicolumn{1}{c}{\textbf{Throughput}} & \multicolumn{1}{c}{\textbf{Accuracy}} & \multicolumn{1}{c}{\textbf{Throughput}} & \multicolumn{1}{c}{\textbf{Accuracy}} & \multicolumn{1}{c}{\textbf{Throughput}} \\ 
\midrule
\multirow{3}{*}{\begin{tabular}[c]{@{}l@{}}\textbf{GSM8K} \\ \textit{4-shot}\end{tabular}} & Vanilla (baseline) & 56.33 & 6.94 \textcolor{black}{(1.0$\times$)} & 73.16 & 6.97 \textcolor{black}{(1.0$\times$)} & 52.01 & 8.07 \textcolor{black}{(1.0$\times$)} & 35.25 & 7.91 \textcolor{black}{(1.0$\times$)} \\
 & Fast-dLLM~\cite{wu2025fast} & 57.01 & 16.18 \textcolor{black}{(2.3$\times$)} & 72.86 & 18.53 \textcolor{black}{(2.7$\times$)} & 50.80 & 17.10 \textcolor{black}{(2.1$\times$)} & 33.51 & 19.44 \textcolor{black}{(2.5$\times$)} \\
 & \cellcolor{Gray}\textit{DiCo} (ours) & \cellcolor{Gray}\textbf{75.13} & \cellcolor{Gray}\textbf{23.46 \textcolor{black}{(3.4$\times$)}} & \cellcolor{Gray}\textbf{75.21} & \cellcolor{Gray}\textbf{25.63 \textcolor{black}{(3.7$\times$)}} & \cellcolor{Gray}\textbf{63.38} & \cellcolor{Gray}\textbf{31.83 \textcolor{black}{(3.9$\times$)}} & \cellcolor{Gray}\textbf{62.93} & \cellcolor{Gray}\textbf{30.12 \textcolor{black}{(3.8$\times$)}} \\ 
\midrule
\multirow{3}{*}{\begin{tabular}[c]{@{}l@{}}\textbf{Math-500} \\ \textit{4-shot}\end{tabular}} & Vanilla (baseline) & 25.20 & 8.77 \textcolor{black}{(1.0$\times$)} & 37.00 & 8.82 \textcolor{black}{(1.0$\times$)} & 3.80 & 8.76 \textcolor{black}{(1.0$\times$)} & 21.40 & 8.97 \textcolor{black}{(1.0$\times$)} \\
 & Fast-dLLM~\cite{wu2025fast} & 25.00 & 17.65 \textcolor{black}{(2.0$\times$)} & 37.00 & 19.44 \textcolor{black}{(2.2$\times$)} & 3.80 & 10.57 \textcolor{black}{(1.2$\times$)} & 22.00 & 18.06 \textcolor{black}{(2.0$\times$)} \\
 & \cellcolor{Gray}\textit{DiCo} (ours) & \cellcolor{Gray}\textbf{36.60} & \cellcolor{Gray}\textbf{23.19 \textcolor{black}{(2.6$\times$)}} & \cellcolor{Gray}\textbf{38.00} & \cellcolor{Gray}\textbf{25.89 \textcolor{black}{(2.9$\times$)}} & \cellcolor{Gray}\textbf{26.20} & \cellcolor{Gray}\textbf{20.72 \textcolor{black}{(2.4$\times$)}} & \cellcolor{Gray}\textbf{36.40} & \cellcolor{Gray}\textbf{25.83 \textcolor{black}{(2.9$\times$)}} \\ 
\midrule
\multirow{3}{*}{\begin{tabular}[c]{@{}l@{}}\textbf{MBPP} \\ \textit{3-shot}\end{tabular}} & Vanilla (baseline) & 14.60 & 8.23 \textcolor{black}{(1.0$\times$)} & 27.80 & 8.25 \textcolor{black}{(1.0$\times$)} & 34.00 & 8.79 \textcolor{black}{(1.0$\times$)} & 29.80 & 8.95 \textcolor{black}{(1.0$\times$)} \\
 & Fast-dLLM~\cite{wu2025fast} & 14.60 & 21.30 \textcolor{black}{(2.6$\times$)} & 28.00 & 26.66 \textcolor{black}{(3.2$\times$)} & 33.80 & 30.77 \textcolor{black}{(3.5$\times$)} & 29.40 & 23.83 \textcolor{black}{(2.7$\times$)} \\
 & \cellcolor{Gray}\textit{DiCo} (ours) & \cellcolor{Gray}\textbf{31.40} & \cellcolor{Gray}\textbf{31.02 \textcolor{black}{(3.8$\times$)}} & \cellcolor{Gray}\textbf{28.60} & \cellcolor{Gray}\textbf{31.02 \textcolor{black}{(3.8$\times$)}} & \cellcolor{Gray}\textbf{46.80} & \cellcolor{Gray}\textbf{69.71 \textcolor{black}{(7.9$\times$)}} & \cellcolor{Gray}\textbf{54.80} & \cellcolor{Gray}\textbf{68.82 \textcolor{black}{(7.7$\times$)}} \\ 
\midrule
\multirow{3}{*}{\begin{tabular}[c]{@{}l@{}}\textbf{HumanEval} \\ \textit{0-shot}\end{tabular}} & Vanilla (baseline) & 12.80 & 18.36 \textcolor{black}{(1.0$\times$)} & 28.05 & 18.39 \textcolor{black}{(1.0$\times$)} & \textbf{17.07} & 15.98 \textcolor{black}{(1.0$\times$)} & 15.24 & 15.89 \textcolor{black}{(1.0$\times$)} \\
 & Fast-dLLM~\cite{wu2025fast} & 12.80 & 74.62 \textcolor{black}{(4.1$\times$)} & 26.83 & 54.01 \textcolor{black}{(2.9$\times$)} & \textbf{17.07} & 37.94 \textcolor{black}{(2.4$\times$)} & 15.24 & 38.41 \textcolor{black}{(2.4$\times$)} \\
 & \cellcolor{Gray}\textit{DiCo} (ours) & \cellcolor{Gray}\textbf{29.27} & \cellcolor{Gray}\textbf{88.55 \textcolor{black}{(4.8$\times$)}} & \cellcolor{Gray}\textbf{37.20} & \cellcolor{Gray}\textbf{62.94 \textcolor{black}{(3.4$\times$)}} & \cellcolor{Gray}15.24 & \cellcolor{Gray}\textbf{48.84 \textcolor{black}{(3.1$\times$)}} & \cellcolor{Gray}\textbf{23.17} & \cellcolor{Gray}\textbf{90.64 \textcolor{black}{(5.7$\times$)}} \\ 
\midrule
\multirow{3}{*}{\begin{tabular}[c]{@{}l@{}}\textbf{Average}\end{tabular}} & Vanilla & 27.23 & 10.58 \textcolor{black}{(1.0$\times$)} & 41.50 & 10.61 \textcolor{black}{(1.0$\times$)} & 26.72 & 10.40 \textcolor{black}{(1.0$\times$)} & 25.42 & 10.43 \textcolor{black}{(1.0$\times$)} \\
 & Fast-dLLM & 27.35 & 32.44 \textcolor{black}{(3.1$\times$)} & 41.17 & 29.66 \textcolor{black}{(2.8$\times$)} & 26.37 & 24.10 \textcolor{black}{(2.3$\times$)} & 25.04 & 24.94 \textcolor{black}{(2.4$\times$)} \\
 & \cellcolor{Gray}\textit{DiCo} (ours) & \cellcolor{Gray}\textbf{43.10} & \cellcolor{Gray}\textbf{41.56 \textcolor{black}{(3.9$\times$)}} & \cellcolor{Gray}\textbf{44.75} & \cellcolor{Gray}\textbf{36.37 \textcolor{black}{(3.4$\times$)}} & \cellcolor{Gray}\textbf{37.91} & \cellcolor{Gray}\textbf{42.80 \textcolor{black}{(4.1$\times$)}} & \cellcolor{Gray}\textbf{44.33} & \cellcolor{Gray}\textbf{53.85 \textcolor{black}{(5.2$\times$)}} \\ 
\bottomrule[0.125em]
\end{tabular}
}
\label{tab:main-results}
\end{table*}

\subsection{Phase 2: How to Conquer}
\label{sec:phase2-how-to-conquer}

In the \textit{Conquer} phase, \textit{DiCo} further performs adaptive parallel decoding across local clusters as constructed in the previous \textit{Divide} phase. We begin with a theorem that offers theoretical insights into the design of the above adaptive parallel decoding strategy.

\begin{theorem}
\label{theorem:1}
For a given dLLM, let $\mathbf{x} = (\mathbf{x}_0, \ldots, \mathbf{x}_{n-1})$ denote its decoding trajectory of length $n$, where $\mathbf{x}_{i}$ represents the token predicted at position $i$ during decoding. For each position $i \in [0,\ldots,n-1]$, its prediction probability under greedy decoding satisfies:
\begin{equation}
P(\mathbf{x}_i) = 1-\epsilon_i, \quad \mathrm{where} \; 0\le \epsilon_i < 1.
\end{equation}
We define $\epsilon = \max_{0\le i\le n-1} \epsilon_i$. Based on the above, we can derive that the following two properties hold: \ul{(1)} $\lim_{\epsilon \to 0} P_{marginal}(\mathbf{x}) = \lim_{\epsilon \to 0} P_{joint}(\mathbf{x})$ and \ul{(2)} the total variation distance $\Delta_{TVD}(P_{marginal}(\mathbf{x}), P_{joint}(\mathbf{x}))$ between the marginal and joint PMFs across all positions is upper bounded by $n\epsilon$. The detailed proof of Theorem~\ref{theorem:1} is provided in the Appendix.
\end{theorem}

{\faLightbulbO}
\textbf{Inspirations.}
Theorem~\ref{theorem:1}.1 indicates that the product of the marginal PMFs converges to the joint PMF when the dLLM’s predictions are sufficiently confident across all masked positions. This justifies performing parallel decoding across multiple masked positions when their confidence scores are high enough. Theorem~\ref{theorem:1}.2 further establishes a quantitative upper bound on the discrepancy between the marginal and joint PMFs, which provides a theoretical foundation for confidence-based parallel decoding.

\textbf{Adaptive parallel decoding.}
Inspired by Theorem~\ref{theorem:1} and \cite{wu2025fast}, we introduce an adaptive parallel decoding approach for the \textit{Conquer} phase to leverage the high-confidence information within local clusters. Instead of relying on a naive top-$k$ or fixed confidence-based scheme, this approach adaptively determines which masked positions to unmask according to their confidence scores. Formally, the set of tokens to unmask at step $t$ is defined as follows:
\begin{equation}
    S_t = \{i\;|\; (|S_t|+1)(1-c^t(i)) < 1 \},
\end{equation}
where $c^t(i)$ denotes the prediction confidence of the masked position $i$ at step $t$. In this approach, the number of tokens unmasked in parallel within each local cluster at a given step is lower-bounded by the minimum confidence among its tokens. This property aligns well with the local clusters formed in the \textit{Divide} phase, where the confidence of all tokens exceeds the parallel decoding threshold $\tau_2$ and continues to rapidly increase as decoding progresses, as discussed in Observation~III. Furthermore, at each decoding step, \textit{DiCo} also re-evaluates the confidence of masked tokens within each local cluster and its surrounding context. After that, \textit{DiCo} adaptively adjusts the boundary of each local cluster by including new masked tokens with confidence $\ge \tau_2$ and excluding masked tokens that have been unmasked or no longer satisfy this threshold $\tau_2$. This adaptive mechanism ensures that each local cluster continuously tracks high-confidence masked tokens, thus sustaining efficient parallel decoding throughout the \textit{Conquer} phase.

\textbf{Divide and conquer until convergence.}
During the \textit{Conquer} phase, \textit{DiCo} continuously monitors the confidence density within each local cluster. Specifically, \textit{DiCo} terminates the current \textit{Conquer} phase once the confidence density falls below the parallel decoding threshold $\tau_2$, which indicates that the remaining masked tokens are no longer sufficiently confident to support reliable parallel decoding. Upon termination, \textit{DiCo} determines its next phase based on the current unmasking ratio, denoted as $R(\mathbf{x}^{t})$. If $R(\mathbf{x}^{t}) < R_{gate}$---a gating criterion suggesting that dependencies among masked tokens have become globalized---\textit{DiCo} re-enters the \textit{Divide} phase to identify new local clusters under the updated context for further adaptive parallel decoding. Otherwise, \textit{DiCo} directly transitions to the \textit{Finalize} phase to generate the final output.

\subsection{Phase 3: How to Finalize}
\label{sec:phase3-how-to-finalize}

After most masked tokens have been decoded through the \textit{divide-and-conquer} process, the context becomes densely populated, leaving a small number of masked tokens scattered across the sequence. In light of this, a fine-grained decoding strategy that leverages the global context is required to finalize the decoding process and ensure coherent generation. To this end, in the \textit{Finalize} phase, we employ an effective fine-grained compound decoding strategy, which considers both the logit margin and top-1 confidence of each masked token to optimize generation quality and stability.

\textbf{Fine-grained compound decoding.}
For each masked position $i$ at the decoding step $t$, its logit margin $\Delta L_i^t$ is mathematically defined as the difference between its highest and second-highest logits as follows:
\begin{equation}
    \Delta L_i^t = \max(L_i^t) - \text{second\_max}(L_i^t),
\end{equation}
where $L_i^t$ denotes the output logits at masked position $i$ and step $t$. In practice, after each forward pass, \textit{DiCo} first selects a subset of masked tokens from the remaining masked tokens that satisfy $\Delta L_i^t > \tau_{3}$, where $\tau_{3}$ is the logit margin threshold. \textit{DiCo} then unmasks the selected masked tokens in parallel. If no masked tokens meet this criterion (i.e., $\Delta L_i^t > \tau_{3}$), \textit{DiCo} defaults to the top-1 confidence-based decoding. We note that compound decoding can balance between decisiveness and caution in the \textit{Finalize} phase. On the one hand, the logit margin reflects the dLLM’s certainty gap between the most and second-most likely tokens, which can effectively avoid premature or ambiguous token selections. On the other hand, the top-1 confidence-based decoding ensures steady decoding progress even when the logit margin is relatively small. This complementary mechanism enables \textit{DiCo} to finalize the decoding process under the global context while maintaining stability and consistency.
\section{Experiments}
\label{sec:experiments}

\subsection{Experimental Settings}
\label{sec:experimental-settings}

\begin{figure}[t]
  \centering
  \includegraphics[width=1.0\linewidth]{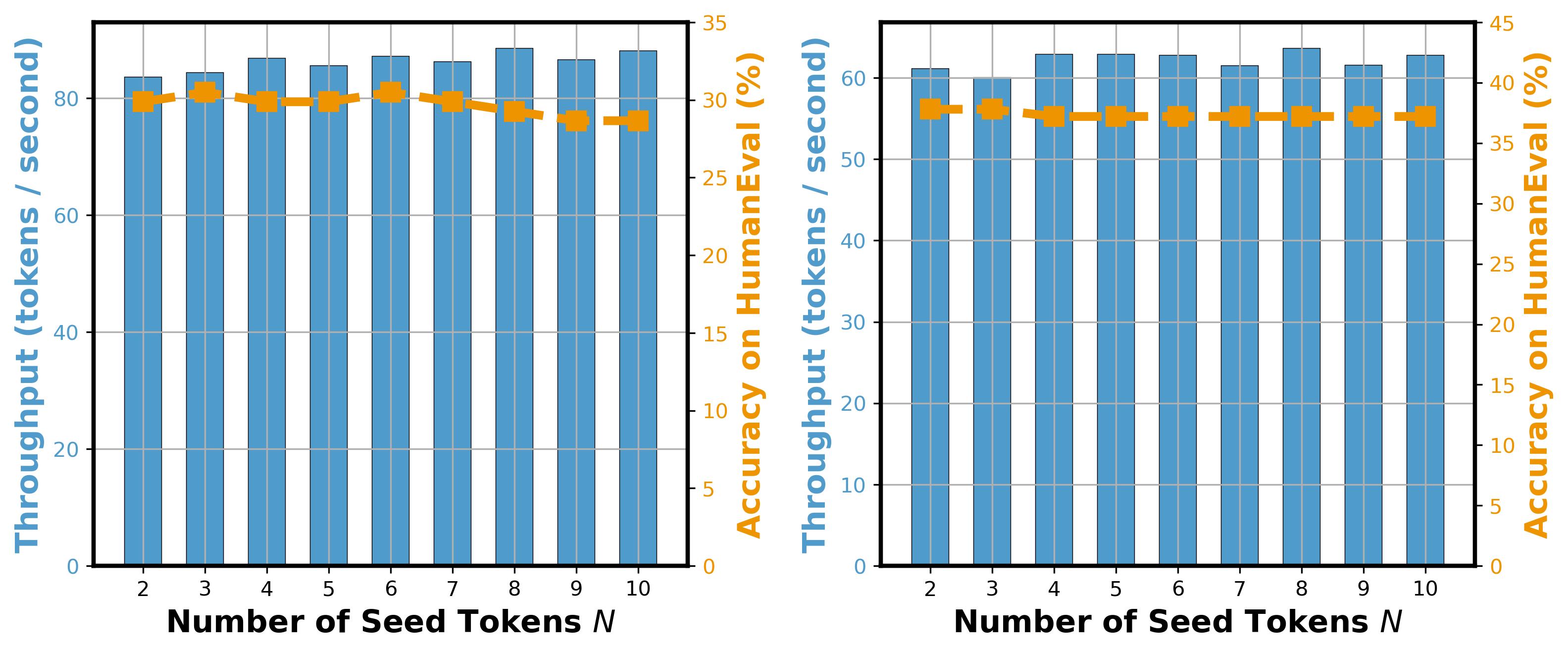} 
  \vspace{-10pt}
  \caption{Ablation results of the number of seek tokens $N$ on HumanEval under the non-AR (\textit{left}) and semi-AR (\textit{right}) settings.}
  \label{fig:abl1} 
\end{figure}

\textbf{Models, datasets, metrics, and hardware.}
Following recent conventions~\cite{wu2025fast}, we evaluate the proposed \textit{DiCo} on two representative dLLMs (i.e., LLaDA-8B-Instruct~\cite{nie2025large} and Dream-7B-Instruct~\cite{ye2025dream}) across four widely used datasets, including two math tasks (i.e., GSM8K~\cite{gsm8k2021} and Math-500~\cite{math5002023}) and two code generation tasks (i.e., HumanEval~\cite{humaneval2021} and MBPP~\cite{mbpp2021}). The performance is reported in terms of \ul{(1)} task accuracy, which is evaluated using the \texttt{lm-eval-harness} framework~\cite{eval-harness}, and \ul{(2)} inference throughput, which is measured by the average number of tokens generated per second (TPS). For fair comparisons, all experiments in this work are conducted on NVIDIA RTX 4090 GPUs (24GB).

\textbf{Baselines.}
For each dLLM, we consider two baselines: \ul{(1)} Vanilla, which performs top-1 confidence-based decoding, and \ul{(2)} Fast-dLLM~\cite{wu2025fast}, which performs confidence-based parallel decoding. For Fast-dLLM, the confidence threshold is set to 0.95. For fair and comprehensive comparisons, we evaluate these two baselines and \textit{DiCo} under both non-AR and semi-AR decoding settings. In the non-AR decoding setting, we set the generation length to 256. In the semi-AR decoding setting, we set the generation length to 256 and the block size to 128.

\textbf{Implementation details.}
For the hyperparameters in \textit{DiCo}, we empirically set $\alpha=0.5$, $\beta=0.05$, $R_{gate}=0.8$, $\tau_1=0.3$, $\tau_2=0.6$, and $\tau_3=3$. Furthermore, in the \textit{Divide} phase, the number of seed tokens $N$ is set to 8 in the non-AR setting and 4 in the semi-AR setting.

\begin{figure}[t]
  \centering
  \includegraphics[width=1.0\linewidth]{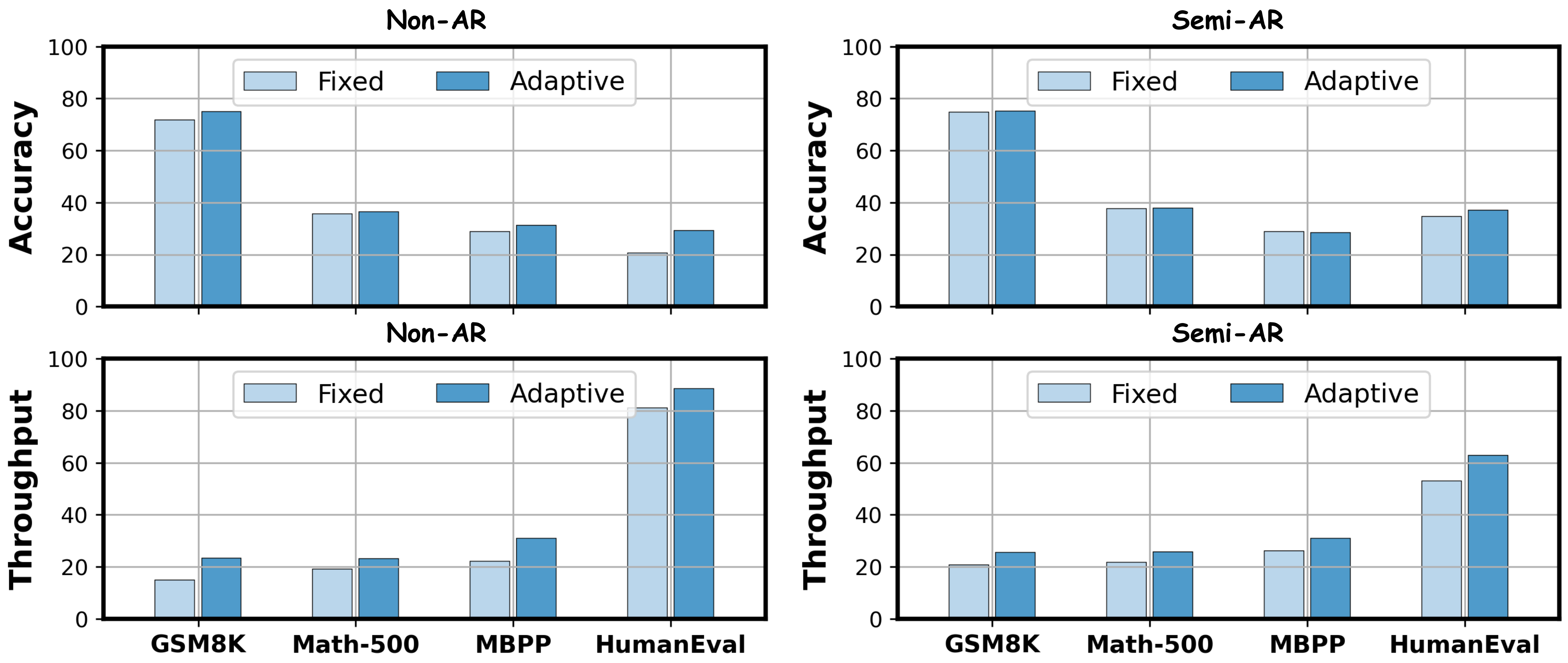} 
  \vspace{-10pt}
  \caption{Fixed confidence-based parallel decoding vs. adaptive parallel decoding under non-AR (\textit{left}) / semi-AR (\textit{right}) settings.}
  \label{fig:abl:fixed_vs_factor} 
\end{figure}

\subsection{Main Results}
\label{sec:main-results}

The main results are summarized in Table~\ref{tab:main-results}, which clearly shows the superiority of \textit{DiCo} on both LLaDA-8B-Instruct and Dream-7B-Instruct across all four datasets. We observe that \textit{DiCo} achieves substantially better accuracy-efficiency trade-offs than Vanilla and Fast-dLLM~\cite{wu2025fast} under both non-AR and semi-AR settings. Taking LLaDA-8B-Instruct as an example, compared to Vanilla, \textit{DiCo} achieves up to 4.8$\times$ inference speedups and $+$18.8\% higher accuracy under the non-AR setting, as well as up to 3.8$\times$ inference speedups and $+$9.15\% higher accuracy under the semi-AR setting. Compared to Fast-dLLM~\cite{wu2025fast}, \textit{DiCo} also achieves much better performance in terms of both accuracy and inference throughput, especially under the non-AR setting. These experimental results demonstrate that \textit{DiCo}, as an adaptive parallel decoding approach, not only delivers significant improvement in inference efficiency, but also unleashes the capability of non-AR decoding for dLLMs---an ability often compromised under the semi-AR setting in prior works~\cite{nie2025large, wu2025fast, wu2025fast-dllmv2}.

\subsection{Ablations and Discussions}
\label{sec:ablations-and-discussions}

In this section, we conduct a series of ablation experiments to analyze \textit{DiCo}, where we leverage LLaDA-8B-Instruct~\cite{nie2025large} as the default dLLM.

\textbf{Ablation on the number of seed tokens $N$.}
We first explore how the hyperparameter $N$---which specifies the number of seed tokens in the \textit{Divide} phase---affects the performance of \textit{DiCo}. The experiment is conducted on HumanEval, an order-sensitive code generation task that is most likely to be affected by the choice of $N$~\cite{humaneval2021}. As shown in Figure~\ref{fig:abl1}, when $N$ increases from 2 to 10, both accuracy and inference throughput remain highly stable under both non-AR and semi-AR settings, indicating that \textit{DiCo} is insensitive to $N$. This reveals that \textit{DiCo} can adaptively determine an appropriate degree of divide-and-conquer through its seed identification mechanism, maintaining stable generation quality in practice.

\textbf{Fixed confidence-based parallel decoding vs. adaptive parallel decoding.}
As discussed in Section~\ref{sec:phase2-how-to-conquer}, we introduce an adaptive parallel decoding strategy in the \textit{Conquer} phase to decode masked tokens across different local clusters in parallel. Below we further compare the adaptive parallel decoding with fixed confidence-based parallel decoding as proposed in Fast-dLLM~\cite{wu2025fast}. As shown in Figure~\ref{fig:abl:fixed_vs_factor}, adaptive parallel decoding consistently achieves better inference throughput across all four datasets, while maintaining comparable accuracy to fixed confidence-based parallel decoding. This clearly demonstrates the efficacy of our proposed adaptive parallel decoding.

\begin{figure}[t]
  \centering
  \includegraphics[width=1.0\linewidth]{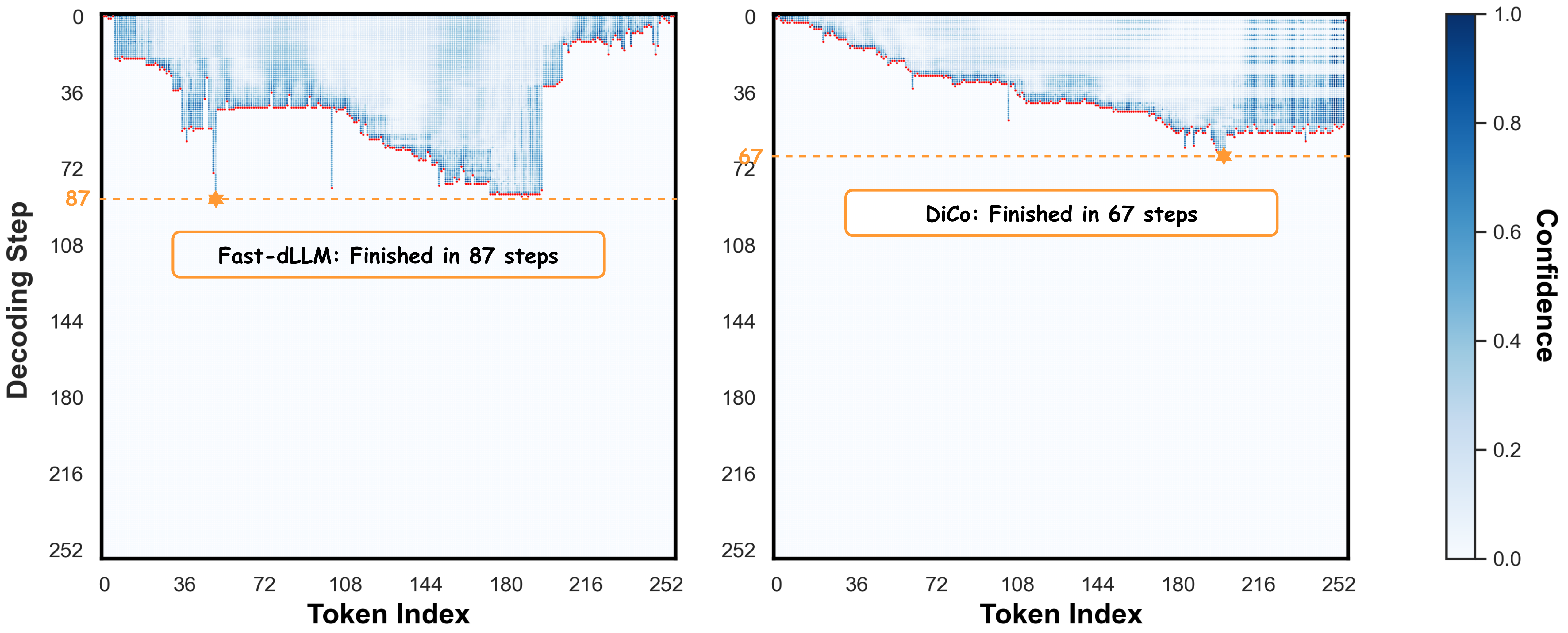} 
  \vspace{-10pt}
  \caption{Visualization of the decoding trajectory of Fast-dLLM (\textit{left}, finished in 87 steps) and \textit{DiCo} (\textit{right}, finished in 67 steps).}
  \label{fig:abl:fastdllm-vs-dico-tg} 
\end{figure}

\textbf{Ablation on trajectory guidance.}
As discussed in Section~\ref{sec:phase1-how-to-divide}, during the \textit{Divide} phase, we introduce trajectory guidance to enhance seed token selection. To further investigate its efficacy, we take GSM8K as an example for analysis. The experimental results are summarized in Table~\ref{tab:abl:tr_gd}. Notably, we observe that without trajectory guidance, \textit{DiCo} achieves higher inference throughput than Vanilla but much lower accuracy on GSM8K. When trajectory guidance is applied, \textit{DiCo} not only surpasses Vanilla by a large margin in accuracy but also retains a notable inference speedup. These results indicate that trajectory guidance effectively aligns seed token selection, thus enhancing generation quality at a minor cost of decoding efficiency.

\begin{table}[h]
\centering
\caption{Ablation results of trajectory guidance. In this table, \texttt{TG} is the abbreviation of trajectory guidance.}
\resizebox{1.0\linewidth}{!}{%
\begin{tabular}{l c c c c}
\toprule[0.125em]
\multirow{2}{*}{\textbf{Method}} & \multicolumn{2}{c}{\textbf{non-AR}} & \multicolumn{2}{c}{\textbf{semi-AR}} \\
\cmidrule(lr){2-3} \cmidrule(lr){4-5}
 & \textbf{Accuracy} & \textbf{Throughput} & \textbf{Accuracy} & \textbf{Throughput} \\ 
\midrule 
Vanilla  & 56.33          & 6.94 (1.0$\times$) & 73.16          & 6.97 (1.0$\times$) \\
\textit{DiCo} w/o \texttt{TG}  & 49.28          & \textbf{34.00 (4.9$\times$)} & 68.92          & \textbf{29.60 (4.2$\times$)} \\
\cellcolor{Gray}\textit{DiCo} w/ \texttt{TG} & \cellcolor{Gray}\textbf{75.13} & \cellcolor{Gray}23.46 (3.4$\times$)          & \cellcolor{Gray}\textbf{75.21} & \cellcolor{Gray}25.63 (3.7$\times$)          \\
\bottomrule[0.125em]
\end{tabular}%
}
\label{tab:abl:tr_gd}
\end{table}

\begin{table}[t]
\centering
\caption{Ablation results of fine-grained compound decoding. In this table, \texttt{LM} is the abbreviation of logit margin.}
\resizebox{1.0\linewidth}{!}{%
\begin{tabular}{l c c c c}
\toprule[0.125em]
\multirow{2}{*}{\textbf{Method}} & \multicolumn{2}{c}{\textbf{non-AR}} & \multicolumn{2}{c}{\textbf{semi-AR}} \\
\cmidrule(lr){2-3} \cmidrule(lr){4-5}
 & \textbf{Accuracy} & \textbf{Throughput} & \textbf{Accuracy} & \textbf{Throughput} \\ 
\midrule 
Vanilla  & 56.33          & 6.94 (1.0$\times$) & 73.16          & 6.97 (1.0$\times$) \\
\textit{DiCo} w/o \texttt{LM} & 74.45 & 19.45 (2.8$\times$) & 75.13 & 22.53(3.2$\times$) \\
\cellcolor{Gray}\textit{DiCo} w/ \texttt{LM} & \cellcolor{Gray}\textbf{75.13} & \cellcolor{Gray}\textbf{23.46 (3.4$\times$)} & \cellcolor{Gray}\textbf{75.21} & \cellcolor{Gray}\textbf{25.63(3.7$\times$)} \\
\bottomrule[0.125em]
\end{tabular}%
}
\label{tab:ablation-logit-margin}
\end{table}

\textbf{Visualization of decoding trajectories.}
To gain further insights, we visualize a representative decoding trajectory of \textit{DiCo} and compare it with Fast-dLLM~\cite{wu2025fast} on the same input sampled from GSM8K. As shown in Figure~\ref{fig:abl:fastdllm-vs-dico-tg}, \textit{DiCo} completes decoding in only 67 steps, whereas Fast-dLLM requires 87 steps. In contrast, Vanilla requires 256 steps (see Figure~\ref{fig:decoding-cluster-observation}). This indicates that \textit{DiCo} can reduce the total number of decoding steps, thus offering better inference speedups. In addition, we observe that \textit{DiCo} enforces the decoding trajectory to approximately follow a left-to-right progression. As noted in~\cite{huang2025pc}, such left-to-right behavior is particularly beneficial for order-sensitive reasoning tasks to achieve better accuracy.

\textbf{Ablation on fine-grained compound decoding.}
As discussed in Section~\ref{sec:phase3-how-to-finalize}, we tailor a fine-grained compound decoding strategy in the \textit{Finalize} phase, which jointly considers the top-1 confidence and logit margin to enhance generation quality and stability. As shown in Table~\ref{tab:ablation-logit-margin}, compared to the default top-1 confidence-based decoding, our \textit{logit margin + top-1 confidence}-based fine-grained compound decoding achieves significantly higher accuracy on GSM8K under both non-AR and semi-AR settings, while maintaining higher inference throughput. This indicates that our fine-grained compound decoding effectively improves the robustness of token generation in the \textit{Finalize} phase by leveraging both the top-1 confidence and logit margin, leading to more reliable final output.
\section{Conclusion}
\label{sec:conclusion}

In this work, we introduce a training-free adaptive parallel decoding approach, dubbed \textit{DiCo}, to bridge the gap between the theoretical parallelism and practical performance of dLLMs. Specifically, \textit{DiCo} features a \textit{divide-and-conquer} paradigm, which consists of three iterative phases: the \textit{Divide} phase for identifying seed tokens and forming local clusters, the \textit{Conquer} phase for performing adaptive parallel decoding across different local clusters, and the \textit{Finalize} phase for decoding the remaining masked tokens and generating the final output. Extensive experiments demonstrate that \textit{DiCo} can achieve substantial inference speedups without compromising generation quality.

\bibliography{reference}

@article{dubey2024llama,
  title={{The Llama 3 Herd of Models}},
  author={Dubey, Abhimanyu and Jauhri, Abhinav and Pandey, Abhinav and Kadian, Abhishek and Al-Dahle, Ahmad and Letman, Aiesha and Mathur, Akhil and Schelten, Alan and Yang, Amy and Fan, Angela and others},
  journal={arXiv e-prints},
  pages={arXiv--2407},
  year={2024}
}

@article{guo2025deepseek,
  title={{DeepSeek-R1 Incentivizes Reasoning in LLMs through Reinforcement Learning}},
  author={Guo, Daya and Yang, Dejian and Zhang, Haowei and Song, Junxiao and Wang, Peiyi and Zhu, Qihao and Xu, Runxin and Zhang, Ruoyu and Ma, Shirong and Bi, Xiao and others},
  journal={Nature},
  volume={645},
  number={8081},
  pages={633--638},
  year={2025},
  publisher={Nature Publishing Group UK London}
}

@article{achiam2023gpt,
  title={{GPT-4 Technical Report}},
  author={Achiam, Josh and Adler, Steven and Agarwal, Sandhini and Ahmad, Lama and Akkaya, Ilge and Aleman, Florencia Leoni and Almeida, Diogo and Altenschmidt, Janko and Altman, Sam and Anadkat, Shyamal and others},
  journal={arXiv preprint arXiv:2303.08774},
  year={2023}
}

@article{li2024large,
  title={{Large Language Model Inference Acceleration: A Comprehensive Hardware Perspective}},
  author={Li, Jinhao and Xu, Jiaming and Huang, Shan and Chen, Yonghua and Li, Wen and Liu, Jun and Lian, Yaoxiu and Pan, Jiayi and Ding, Li and Zhou, Hao and others},
  journal={arXiv preprint arXiv:2410.04466},
  year={2024}
}

@article{li2024survey,
  title={{A Survey on Large Language Model Acceleration based on KV Cache Management}},
  author={Li, Haoyang and Li, Yiming and Tian, Anxin and Tang, Tianhao and Xu, Zhanchao and Chen, Xuejia and Hu, Nicole and Dong, Wei and Li, Qing and Chen, Lei},
  journal={arXiv preprint arXiv:2412.19442},
  year={2024}
}

@article{sahoo2024simple,
  title={{Simple and Effective Masked Diffusion Language Models}},
  author={Sahoo, Subham and Arriola, Marianne and Schiff, Yair and Gokaslan, Aaron and Marroquin, Edgar and Chiu, Justin and Rush, Alexander and Kuleshov, Volodymyr},
  journal={Advances in Neural Information Processing Systems},
  volume={37},
  pages={130136--130184},
  year={2024}
}

@article{liu2025dllm-cache,
  title={dLLM-Cache: Accelerating Diffusion Large Language Models with Adaptive Caching},
  author={Liu, Zhiyuan and Yang, Yicun and Zhang, Yaojie and Chen, Junjie and Zou, Chang and Wei, Qingyuan and Wang, Shaobo and Zhang, Linfeng},
  journal={arXiv preprint arXiv:2506.06295},
  year={2025}
}

@article{ma2025dkv-cache,
  title={dKV-Cache: The Cache for Diffusion Language Models},
  author={Ma, Xinyin and Yu, Runpeng and Fang, Gongfan and Wang, Xinchao},
  journal={arXiv preprint arXiv:2505.15781},
  year={2025}
}

@article{jiang2025d,
  title={d$^{2}$Cache: Accelerating Diffusion-Based LLMs via Dual Adaptive Caching},
  author={Jiang, Yuchu and Cai, Yue and Luo, Xiangzhong and Fu, Jiale and Wang, Jiarui and Liu, Chonghan and Yang, Xu},
  journal={arXiv preprint arXiv:2509.23094},
  year={2025}
}

@article{nie2025large,
  title={{Large Language Diffusion Models}},
  author={Nie, Shen and Zhu, Fengqi and You, Zebin and Zhang, Xiaolu and Ou, Jingyang and Hu, Jun and Zhou, Jun and Lin, Yankai and Wen, Ji-Rong and Li, Chongxuan},
  journal={arXiv preprint arXiv:2502.09992},
  year={2025}
}

@article{ye2025dream,
  title={{Dream 7B: Diffusion Large Language Models}},
  author={Ye, Jiacheng and Xie, Zhihui and Zheng, Lin and Gao, Jiahui and Wu, Zirui and Jiang, Xin and Li, Zhenguo and Kong, Lingpeng},
  journal={arXiv preprint arXiv:2508.15487},
  year={2025}
}

@article{li2025survey-dllm,
  title={{A Survey on Diffusion Language Models}},
  author={Li, Tianyi and Chen, Mingda and Guo, Bowei and Shen, Zhiqiang},
  journal={arXiv preprint arXiv:2508.10875},
  year={2025}
}

@article{israel2025accelerating,
  title={{Accelerating Diffusion LLMs via Adaptive Parallel Decoding}},
  author={Israel, Daniel and Broeck, Guy Van den and Grover, Aditya},
  journal={arXiv preprint arXiv:2506.00413},
  year={2025}
}

@article{bao2025learning,
  title={{Learning to Parallel: Accelerating Diffusion Large Language Models via Adaptive Parallel Decoding}},
  author={Bao, Wenrui and Chen, Zhiben and Xu, Dan and Shang, Yuzhang},
  journal={arXiv preprint arXiv:2509.25188},
  year={2025}
}

@article{chen2025dparallel,
  title={{dParallel: Learnable Parallel Decoding for dLLMs}},
  author={Chen, Zigeng and Fang, Gongfan and Ma, Xinyin and Yu, Ruonan and Wang, Xinchao},
  journal={arXiv preprint arXiv:2509.26488},
  year={2025}
}

@article{wu2025fast,
  title={{Fast-dLLM: Training-free Acceleration of Diffusion LLM by Enabling KV Cache and Parallel Decoding}},
  author={Wu, Chengyue and Zhang, Hao and Xue, Shuchen and Liu, Zhijian and Diao, Shizhe and Zhu, Ligeng and Luo, Ping and Han, Song and Xie, Enze},
  journal={arXiv preprint arXiv:2505.22618},
  year={2025}
}

@article{wu2025fast-dllmv2,
  title={{Fast-dLLM v2: Efficient Block-Diffusion LLM}},
  author={Wu, Chengyue and Zhang, Hao and Xue, Shuchen and Diao, Shizhe and Fu, Yonggan and Liu, Zhijian and Molchanov, Pavlo and Luo, Ping and Han, Song and Xie, Enze},
  journal={arXiv preprint arXiv:2509.26328},
  year={2025}
}

@article{wu2025free,
  title={{Free Draft-and-Verification: Toward Lossless Parallel Decoding for Diffusion Large Language Models}},
  author={Wu, Shutong and Zhang, Jiawei},
  journal={arXiv preprint arXiv:2510.00294},
  year={2025}
}

@article{huang2025pc,
  title={{PC-Sampler: Position-Aware Calibration of Decoding Bias in Masked Diffusion Models}},
  author={Huang, Pengcheng and Liu, Shuhao and Liu, Zhenghao and Yan, Yukun and Wang, Shuo and Chen, Zulong and Xiao, Tong},
  journal={arXiv preprint arXiv:2508.13021},
  year={2025}
}

@article{croitoru2023diffusion-survey,
  title={{Diffusion Models in Vision: A Survey}},
  author={Croitoru, Florinel-Alin and Hondru, Vlad and Ionescu, Radu Tudor and Shah, Mubarak},
  journal={IEEE Transactions on Pattern Analysis and Machine Intelligence},
  volume={45},
  number={9},
  pages={10850--10869},
  year={2023},
  publisher={Ieee}
}

@article{nie2024scaling,
  title={{Scaling up Masked Diffusion Models on Text}},
  author={Nie, Shen and Zhu, Fengqi and Du, Chao and Pang, Tianyu and Liu, Qian and Zeng, Guangtao and Lin, Min and Li, Chongxuan},
  journal={arXiv preprint arXiv:2410.18514},
  year={2024}
}

@article{shi2024simplified,
  title={{Simplified and Generalized Masked Diffusion for Discrete Data}},
  author={Shi, Jiaxin and Han, Kehang and Wang, Zhe and Doucet, Arnaud and Titsias, Michalis},
  journal={Advances in neural information processing systems},
  volume={37},
  pages={103131--103167},
  year={2024}
}

@article{zhu2025llada-v1.5,
  title={{LLaDA 1.5: Variance-Reduced Preference Optimization for Large Language Diffusion Models}},
  author={Zhu, Fengqi and Wang, Rongzhen and Nie, Shen and Zhang, Xiaolu and Wu, Chunwei and Hu, Jun and Zhou, Jun and Chen, Jianfei and Lin, Yankai and Wen, Ji-Rong and others},
  journal={arXiv preprint arXiv:2505.19223},
  year={2025}
}

@article{zhu2025llada-moe,
  title={{LLaDA-MoE: A Sparse MoE Diffusion Language Model}},
  author={Zhu, Fengqi and You, Zebin and Xing, Yipeng and Huang, Zenan and Liu, Lin and Zhuang, Yihong and Lu, Guoshan and Wang, Kangyu and Wang, Xudong and Wei, Lanning and others},
  journal={arXiv preprint arXiv:2509.24389},
  year={2025}
}

@article{gao2025self,
  title={{Self Speculative Decoding for Diffusion Large Language Models}},
  author={Gao, Yifeng and Ji, Ziang and Wang, Yuxuan and Qi, Biqing and Xu, Hanlin and Zhang, Linfeng},
  journal={arXiv preprint arXiv:2510.04147},
  year={2025}
}

@inproceedings{leviathan2023fast,
  title={{Fast Inference from Transformers via Speculative Decoding}},
  author={Leviathan, Yaniv and Kalman, Matan and Matias, Yossi},
  booktitle={International Conference on Machine Learning},
  pages={19274--19286},
  year={2023},
  organization={PMLR}
}

@inproceedings{softnms2017,
  title={{Soft-NMS -- Improving Object Detection With One Line of Code}},
  author={Bodla, Navaneeth and Singh, Bharat and Chellappa, Rama and Davis, Larry S},
  booktitle={Proceedings of the IEEE International Conference on Computer Vision},
  pages={5561--5569},
  year={2017}
}

@book{LevinPeres2017,
  title     = {{Markov Chains and Mixing Times}},
  author    = {Levin, David A. and Peres, Yuval},
  year      = {2017},
  publisher = {American Mathematical Society},
  address   = {Providence, RI},
  edition   = {Second},
  isbn      = {978-1-4704-2962-1}
}

@article{gsm8k2021,
  title={{Training Verifiers to Solve Math Word Problems}},
  author={Cobbe, Karl and Kosaraju, Vineet and Bavarian, Mohammad and Chen, Mark and Jun, Heewoo and Kaiser, Lukasz and Plappert, Matthias and Tworek, Jerry and Hilton, Jacob and Nakano, Reiichiro and others},
  journal={arXiv preprint arXiv:2110.14168},
  year={2021}
}

@article{mbpp2021,
  title={{Program Synthesis with Large Language Models}},
  author={Austin, Jacob and Odena, Augustus and Nye, Maxwell and Bosma, Maarten and Michalewski, Henryk and Dohan, David and Jiang, Ellen and Cai, Carrie and Terry, Michael and Le, Quoc and others},
  journal={arXiv preprint arXiv:2108.07732},
  year={2021}
}

@article{humaneval2021,
  title={{Evaluating Large Language Models Trained on Code}},
  author={Chen, Mark and Tworek, Jerry and Jun, Heewoo and Yuan, Qiming and Pinto, Henrique Ponde De Oliveira and Kaplan, Jared and Edwards, Harri and Burda, Yuri and Joseph, Nicholas and Brockman, Greg and others},
  journal={arXiv preprint arXiv:2107.03374},
  year={2021}
}

@inproceedings{math5002023,
  title={{Let's Verify Step by Step}},
  author={Lightman, Hunter and Kosaraju, Vineet and Burda, Yuri and Edwards, Harrison and Baker, Bowen and Lee, Teddy and Leike, Jan and Schulman, John and Sutskever, Ilya and Cobbe, Karl},
  booktitle={The Twelfth International Conference on Learning Representations},
  year={2023}
}

@misc{eval-harness,
  author       = {Gao, Leo and Tow, Jonathan and Abbasi, Baber and Biderman, Stella and Black, Sid and DiPofi, Anthony and Foster, Charles and Golding, Laurence and Hsu, Jeffrey and Le Noac'h, Alain and Li, Haonan and McDonell, Kyle and Muennighoff, Niklas and Ociepa, Chris and Phang, Jason and Reynolds, Laria and Schoelkopf, Hailey and Skowron, Aviya and Sutawika, Lintang and Tang, Eric and Thite, Anish and Wang, Ben and Wang, Kevin and Zou, Andy},
  title        = {{The Language Model Evaluation Harness}},
  month        = 07,
  year         = 2024,
  publisher    = {Zenodo},
  version      = {v0.4.3},
  doi          = {10.5281/zenodo.12608602},
  url          = {https://zenodo.org/records/12608602}
}
\bibliographystyle{icml2025}

\clearpage
\newpage
\appendix
\onecolumn
\section{Proof of Theorem 4.1}
\label{sec:appendix}

\textbf{Notations.}
Let $\mathbf{x}=(\mathbf{x}_0,\ldots,\mathbf{x}_{n-1})$ denote a decoding trajectory produced by a dLLM and $\mathbf{x}^* = (\mathbf{x}^*_{0},\ldots,\mathbf{x}^*_{n-1})$ denote the trajectory produced by greedy decoding. Let $p(\cdot)$ and $q(\cdot)$ denote $P_{marginal}(\cdot)$ and $P_{joint}(\cdot)$. Finally, we use $\mathbf{x}^c$ to denote the complement of $\mathbf{x}$.

\begin{proof}\label{proof:1}
\textbf{Property (1)}: $\lim_{\epsilon \to 0} p(\mathbf{x}) = \lim_{\epsilon \to 0} q(\mathbf{x})$.

\noindent
\ul{Case 1: $\mathbf{x} = \mathbf{x}^*$.} 
In this case, the product of marginal PMFs $p(\mathbf{x})$ is:
\begin{equation*}
p(\mathbf{x})
= \prod\nolimits_{i=0}^{n-1} P(\mathbf{x}_i)
= \prod\nolimits_{i=0}^{n-1} (1-\epsilon_i).
\end{equation*}
For the joint PMF $q(\mathbf{x})$, using the union bound, we have:
\begin{equation*}
1 - n\epsilon 
\le 1 - \sum\nolimits_{i=0}^{n-1} P(\mathbf{x}_i^c)
\le 1 - P(\mathbf{x}^c) 
= q(\mathbf{x})
\le 1.
\end{equation*}
Therefore, when $\epsilon \to 0$, we can derive that both $p(\mathbf{x})$ and $q(\mathbf{x})$ converge to 1, i.e., $\lim_{\epsilon \rightarrow 0} p(\mathbf{x}) = \lim_{\epsilon \rightarrow 0} q(\mathbf{x}) = 1$.

\noindent
\ul{Case 2: $\mathbf{x} \ne \mathbf{x}^*$.} 
In this case, there must exist at least one position $0 \le k \le n-1$ such that $\mathbf{x}_k \ne \mathbf{x}_k^*$. For that position, we have:
\begin{equation*}
P(\mathbf{x}_k) \le 1 - P(\mathbf{x}^*_k) = \epsilon_k \le \epsilon.
\end{equation*}
Therefore, the product of marginal PMFs $p(\mathbf{x})$ satisfies:
\begin{equation*}
0
\le
p(\mathbf{x})
= \prod\nolimits_{i=0}^{n-1} P(\mathbf{x}_i)
= P(\mathbf{x}_k) \prod\nolimits_{i=0}^{k-1} P(\mathbf{x}_i) \prod\nolimits_{j=k+1}^{n-1} P(\mathbf{x}_j)
\le \epsilon.
\end{equation*}
For the joint PMF $q(\mathbf{x})$, using the monotonicity, we have:
\begin{equation*}
0
\le
q(\mathbf x)
= P(\mathbf{x}_0,\ldots,\mathbf{x}_{n-1})
\le P(\mathbf{x}_{k})
\le \epsilon.
\end{equation*}
Therefore, when $\epsilon \to 0$, we can derive that both $p(\mathbf{x})$ and $q(\mathbf{x})$ converge to 0, i.e., $\lim_{\epsilon \rightarrow 0} p(\mathbf{x}) = \lim_{\epsilon \rightarrow 0} q(\mathbf{x}) = 0$.

\noindent
\ul{Combining both cases}, $\lim_{\epsilon\to0} p(\mathbf{x}) = \lim_{\epsilon\to0} q(\mathbf{x})$ holds for any decoding trajectory $\mathbf{x}$. In other words, the product of marginal PMFs converges to the joint PMF when $\epsilon$ is sufficiently small.
\end{proof}

\begin{proof} \label{proof:2}
\textbf{Property (2)}: $\Delta_{TVD}\big(p(\mathbf{x}), q(\mathbf{x})\big) \le n\epsilon$.

\noindent
Following the definition of total variation distance~\cite{LevinPeres2017}, we have:
\begin{equation*}
\begin{split}
    \Delta_{TVD}\big(p(\mathbf{x}),q(\mathbf{x})\big) 
    &= \frac{1}{2} \sum\nolimits_{\mathbf{x}} |p(\mathbf{x}) - q(\mathbf{x})|\\
    &= \frac{1}{2} \big(|p(\mathbf{x}^*) - q(\mathbf{x}^*)| + \sum\nolimits_{\mathbf{x} \ne \mathbf{x}^*} |p(\mathbf{x}) - q(\mathbf{x})|\big)\\
    &= \frac{1}{2} \big(|p({\mathbf{x}^*}^c) - q({\mathbf{x}^*}^c)| + \sum\nolimits_{\mathbf{x} \ne \mathbf{x}^*} |p(\mathbf{x}) - q(\mathbf{x})|\big)\\
    &\le \frac{1}{2} \big(|p({\mathbf{x}^*}^c) - q({\mathbf{x}^*}^c)| + \sum\nolimits_{\mathbf{x} \ne \mathbf{x}^*} (p(\mathbf{x}) + q(\mathbf{x}))\big)\\
    &= \frac{1}{2} \big(|p({\mathbf{x}^*}^c) - q({\mathbf{x}^*}^c)| + p({\mathbf{x}^*}^c) + q({\mathbf{x}^*}^c)\big).
\end{split}    
\end{equation*}
If $p({\mathbf{x}^*}^c) \ge q({\mathbf{x}^*}^c)$, then $\Delta_{TVD}(p(\mathbf{x}),q(\mathbf{x})) \le p({\mathbf{x}^*}^c)$ holds. Otherwise, $\Delta_{TVD}(p(\mathbf{x}),q(\mathbf{x})) \le q({\mathbf{x}^*}^c)$ holds. Thus, $\Delta_{TVD}(p(\mathbf{x}),q(\mathbf{x})) \le \max \big(p({\mathbf{x}^*}^c), q({\mathbf{x}^*}^c)\big)$ holds. It can also be easily derived that $q({\mathbf{x}^*}^c)
\le \sum\nolimits_{i=0}^{n-1} P({\mathbf{x}_i^*}^c)
= \sum\nolimits_{i=0}^{n-1} \epsilon_i$. As for $p({\mathbf{x}^*}^c) 
= 1 - p(\mathbf{x}^*)
= 1 - \prod_{i=0}^{n-1}(1 - \epsilon_i)$, according to the Bernoulli inequality $\prod\nolimits_{i=0}^{n-1}(1 - \epsilon_i)
\ge 1 - \sum\nolimits_{i=0}^{n-1} \epsilon_i$, we can derive that $p({\mathbf{x}^*}^c) \le \sum_{i=0}^{n-1} \epsilon_i$. Taken together, we can derive that $\Delta_{TVD}(p(\mathbf{x}),q(\mathbf{x}))
\le \max\big(p({\mathbf{x}^*}^c),\, q({\mathbf{x}^*}^c)\big)
\le \sum_{i=0}^{n-1} \epsilon_i
\le n\epsilon$.
\end{proof}

\end{document}